\documentclass[letterpaper]{article} 
\usepackage{aaai2026}  
\usepackage{times}  
\usepackage{helvet}  
\usepackage{courier}  
\usepackage[hyphens]{url}  
\usepackage{graphicx} 
\usepackage{natbib}  
\usepackage{caption} 
\usepackage{amsmath}
\usepackage{amssymb}
\usepackage{amsthm}
\usepackage{subcaption}
\usepackage{enumitem}
\usepackage{booktabs}
\usepackage{xcolor}

\title{Koopman Invariants as Drivers of Emergent Time-Series\\Clustering in Joint-Embedding Predictive Architectures}
\author{
    Pablo Ruiz-Morales\textsuperscript{\rm 1,\rm 3}, 
    Dries Vanoost\textsuperscript{\rm 2,\rm 3}, 
    Davy Pissoort\textsuperscript{\rm 2,\rm 3}, 
    Mathias Verbeke\textsuperscript{\rm 1,\rm 3}
}
\affiliations{
    \textsuperscript{\rm 1}Declarative Languages and Artificial Intelligence (DTAI), M-Group, KU Leuven, Bruges, Belgium\\
    \textsuperscript{\rm 2}ESAT-WaveCore, M-Group, KU Leuven, Bruges, Belgium\\
    \textsuperscript{\rm 3}Flanders Make@KU Leuven, Belgium\\
    
    \{pablo.ruizmorales, dries.vanoost, davy.pissoort, mathias.verbeke\}@kuleuven.be
}

\newcommand{\R}{\mathbb{R}}
\newcommand{\N}{\mathbb{N}}
\newcommand{\C}{\mathbb{C}}
\newcommand{\E}{\mathbb{E}}
\newcommand{\cX}{\mathcal{X}}
\newcommand{\cV}{\mathcal{V}}
\newcommand{\cK}{\mathcal{K}}
\newcommand{\ones}{\mathbf{1}}
\newcommand{\fEMA}{f_{\mathrm{EMA}}}
\newcommand{\fenc}{f_{\theta}}
\newcommand{\fpred}{g_{\phi}}
\newcommand{\paramtheta}{\theta}
\newcommand{\paramphi}{\phi}

\newtheorem{theorem}{Theorem}[section]
\newtheorem{lemma}[theorem]{Lemma}
\newtheorem{assumption}[theorem]{Assumption}

\theoremstyle{definition}
\newtheorem{definition}[theorem]{Definition}

\usepackage{fancyhdr}
\nocopyright
\begin{document}

\maketitle

\begin{abstract}
Joint-Embedding Predictive Architectures (JEPAs), a powerful class of self-supervised models, exhibit an unexplained ability to cluster time-series data by their underlying dynamical regimes. We propose a novel theoretical explanation for this phenomenon, hypothesizing that JEPA's predictive objective implicitly drives it to learn the invariant subspace of the system's Koopman operator. We prove that an idealized JEPA loss is minimized when the encoder represents the system's regime indicator functions, which are Koopman eigenfunctions. This theory was validated on synthetic data with known dynamics, demonstrating that constraining the JEPA's linear predictor to be a near-identity operator is the key inductive bias that forces the encoder to learn these invariants. We further discuss that this constraint is critical for selecting this interpretable solution from a class of mathematically equivalent but entangled optima, revealing the predictor's role in representation disentanglement. This work demystifies a key behavior of JEPAs, provides a principled connection between modern self-supervised learning and dynamical systems theory, and informs the design of more robust and interpretable time-series models.
\end{abstract}


\section{Introduction}
Self-Supervised Learning (SSL) has emerged as a powerful paradigm for learning rich data representations from unlabeled data, driving significant progress across diverse domains \cite{Chen2020SimpleFS, Grill2020BootstrapYO, Chen2021}. These methods learn by solving pretext tasks, with the hope that the learned representations capture underlying structures of the data. However, despite their empirical power, the precise mechanisms by which some SSL architectures discover these structures often remain opaque, treating critical model components as black boxes. This underscores a pressing need for theoretical frameworks that can provide a principled understanding of their behavior and guide future development.

Among the diverse SSL strategies, Joint-Embedding Predictive Architectures (JEPAs) \cite{LeCun2022PathToAutonomousAI, Assran2023SelfSupervisedLA} offer a compelling approach. Rather than reconstructing raw inputs, JEPA is a non-generative approach that learns by predicting future data representations within an abstract, jointly-embedded latent space. This focus on abstract predictability has proven highly effective  across diverse domains, spanning from static images to dynamic video and time-series data \cite{verdenius2024latpfnjointembeddingpredictive, Assran2023SelfSupervisedLA, bardes2024revisiting}. 

Intriguingly, when applied to time-series data, JEPA models often yield latent embeddings that spontaneously cluster by underlying, unannotated dynamical regimes. This behavior is not universally observed in other representation learning frameworks, even those employing encoder architectures of similar capacity but with different objectives (e.g., reconstruction-based autoencoders). For instance, as empirically demonstrated in Section \ref{sec:results} (Figure \ref{fig:actual_jepa_vs_ae_clustering}), JEPA can effectively disentangle distinct dynamical modes where a comparable autoencoder (AE) fails to do so.

This pronounced difference in latent organization immediately poses a crucial question: Why does JEPA's predictive objective lead to this regime-aware clustering? What intrinsic mechanism within JEPA drives this emergence of order from apparently unstructured input? Addressing this question is vital not only for understanding JEPA itself but also for developing more principled and effective SSL methods.

To address this, we turn to dynamical systems theory, specifically the Koopman operator framework \cite{Koopman1931HamiltonianSA, Mezic2005SpectralPO}. The Koopman operator offers a powerful way to analyze nonlinear dynamical systems by lifting observations into a space where their evolution becomes linear. This approach has inspired various machine learning techniques aiming to learn these linear representations. For example, significant research has focused on using deep autoencoders to learn intrinsic coordinates where dynamics evolve linearly under a learned Koopman operator, often with auxiliary networks to enforce desired properties like parsimony or to capture continuous spectra \cite{Lusch2018, Yeung2019}. Architectures like Linearly Recurrent Autoencoder Networks explicitly constrain latent dynamics to follow a linear recurrent layer, effectively learning finite-dimensional Koopman approximations \cite{Azencot2020, Otto2019}. This philosophy of enforcing linear latent dynamics has achieved remarkable success in modern structured State-Space Models (SSMs) like Mamba \cite{gu2022efficientlymodelinglongsequences, gu2024mambalineartimesequencemodeling}, which represent the current state-of-the-art for many sequence modeling tasks. These methods, building upon foundations like Dynamic Mode Decomposition and its extensions \cite{DMD2010Schmid, Korda2018}, demonstrate the utility of Koopman theory for system identification and representation learning where linear evolution is a primary target.

Other approaches, such as VAMPnets \cite{VAMPNets2018Mardt} and Time-Lagged Autoencoders \cite{Wehmeyer2018}, are designed to learn slow collective variables or eigenfunctions of the underlying system's transfer operator, which are crucial for understanding long-timescale dynamics and regime transitions. A distinct but related goal in representation learning is to achieve disentanglement, where latent variables are explicitly regularized to capture independent factors of variation in the data \cite{higgins2017betavae, Kim2018}. While these methods successfully identify key dynamical or generative modes, they often employ specialized objectives directly tied to these targets.

In contrast to these approaches that explicitly model system dynamics, optimize for specific eigenfunctions, or regularize for disentangled factors, the mechanism within JEPA is implicit. While many self-supervised methods for time series rely on contrastive objectives that learn similarity based on temporal proximity \cite{Yue2022}, we hypothesize that JEPA's purely predictive objective in a latent space is what drives it to learn functions that are invariant under the system's evolution within distinct dynamical regimes. 

We will demonstrate that these invariant functions correspond to the indicator functions of these regimes, which are, in fact, eigenfunctions of the system's $\Delta$-step Koopman operator associated with a unit‑magnitude eigenvalue. Thus, while not designed as a Koopman modeling tool, JEPA's core learning principle appears to converge on identifying these fundamental Koopman invariants when distinct, stable dynamical regimes are present.

Our contributions are twofold:
\begin{enumerate}
\item We provide a theoretical derivation showing that an idealized JEPA loss function is minimized when its encoder learns to span the space of these Koopman-invariant regime indicators. This offers a first-principles explanation for the observed clustering.
\item We detail an empirical validation strategy using synthetic time-series data with known underlying regimes. Beyond standard t-SNE visualization, our methodology includes a novel set of analyses focused on the learned linear predictor matrix $M$ by examining its Frobenius norm difference from identity, symmetry, eigenvalue spectrum, and action on empirically derived cluster centroids to directly test the theoretical prediction that $M$ behaves as an identity operator on the learned regime subspace.
\end{enumerate}

By clarifying the mechanism behind JEPA's emergent clustering, this work not only offers a deeper understanding of this powerful SSL architecture but also strengthens the promising connections between modern deep learning and the established mathematics of dynamical systems. 

This paper is structured as follows: Section \ref{sec:framework} introduces JEPA and Koopman operator theory in more detail, as the foundation for the theoretical derivation in Section \ref{sec:theory_main}. Section \ref{sec:experiments} details the experimental setup for validation, followed by the obtained results and discussion in Section \ref{sec:results}. Finally, in Section \ref{sec:conclusion}, we synthesize our findings and outline promising directions for future research.

\section{Theoretical Framework}
\label{sec:framework}
To understand the emergent clustering phenomenon in JEPAs, we first need to formalize its learning process and then introduce the mathematical tools from dynamical systems theory that will allow us to analyze its behavior. Our central aim is to connect JEPA's predictive objective to the discovery of underlying dynamical regimes.

\subsection{JEPA Model and Idealized Loss}
\label{subsec:jepa_formalism}
We consider time-series windows 
$x_{t} = (s_{t}, \dots, s_{t+n-1}) \in \R^d \equiv \cX$, 
drawn from a stationary process with invariant measure $\mu$. The JEPA model consists of three neural network components:

\begin{itemize}
    \item An online encoder $\fenc: \cX \to \R^k$, which maps an input window $x_t$ to a latent representation $z_t=\fenc(x_t)$.
    \item An online predictor $\fpred: \R^k \to \R^k$, which takes $z_t$ and predicts the latent representation of a future window $x_{t+\Delta}$ (where $\Delta \geq 1$ is the prediction horizon).
    \item A target encoder $\fEMA: \cX \to \R^k$, structurally identical to $\fenc$, whose parameters $\theta_{EMA}$ are updated as an Exponential Moving Average (EMA) of the online encoder's parameters: $\theta_{EMA} \xleftarrow{} \alpha\,\theta_{EMA} + (1-\alpha)\,\theta$.
\end{itemize}

\begin{figure}
\centering
\includegraphics[width=\linewidth]{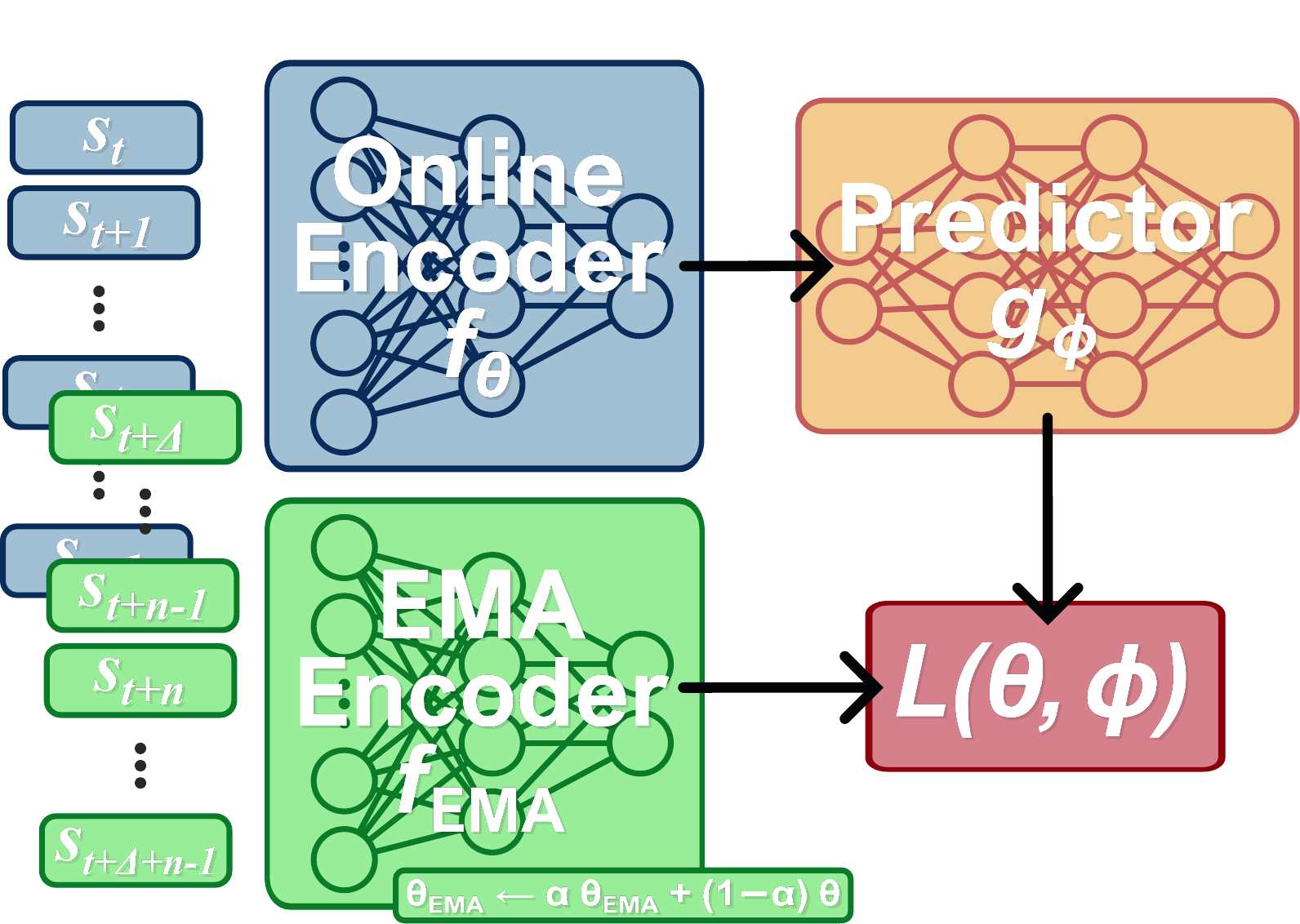}
\caption{Conceptual schematic of the Joint-Embedding Predictive Architecture (JEPA). The online encoder and predictor aim to predict the representation of a future target window as generated by the slowly evolving EMA encoder.}
\label{fig:jepa_schematic}
\end{figure}

The momentum coefficient $\alpha$ is typically chosen to be very close to 1 (e.g., 0.996 to 0.999), ensuring that $\theta_{EMA}$ represents a slowly evolving, stable average of $\theta$ \cite{Grill2020BootstrapYO, He2020}. The target encoder provides stable targets $z_{t+\Delta}=\fEMA(x_{t+\Delta})$ for training.

The online network parameters $\theta$ and $\phi$ are learned by minimizing the predictive loss, typically a squared Euclidean distance, averaged over the data distribution $\mu$ of the input windows:
 \begin{equation}
    L(\paramtheta, \paramphi) = \E_{x_t \sim \mu} \left[ \| \fpred(\fenc(x_t)) - \fEMA(x_{t+\Delta}) \|_2^2 \right]
    \label{eq:jepa_loss_full_framework}
\end{equation}

To perform a tractable theoretical analysis of the representations $\fenc$ JEPA learns, we make a key simplifying assumption. We assume that the target encoder $\fEMA$ closely tracks the online encoder $\fenc$. Specifically, if the rate of change of the online parameters $\theta$ per training step is small relative to $(1-\alpha)$, then the difference $\|\theta-\theta_{EMA}\|$ remains small. Consequently, for well-behaved neural network functions $f$, this implies $\fEMA(y) \approx \fenc(y)$ for any relevant input $y$. This approximation is common in the analysis of similar self-supervised methods employing EMA targets \cite{Grill2020BootstrapYO, He2020, tian2021understanding} and is implicitly validated by their empirical success, which relies on effective target stabilization. 

Under this condition of close tracking, the objective function in \eqref{eq:jepa_loss_full_framework} can be approximated by:
\begin{equation}
    L(f, g) \approx \E_{x_t \sim \mu} \left[ \| g(f(x_t)) - f(x_{t+\Delta}) \|_2^2 \right],
    \label{eq:jepa_loss_idealized_prelim_revised}
\end{equation}
where $f \equiv \fenc$ and $g \equiv \fpred$. Our theoretical development will be based on this idealized loss, while acknowledging that the EMA dynamics introduce a slight deviation in practice.

\subsection{Koopman Operator}
\label{subsec:koopman_formalism}
To analyze the functions $f$ that JEPA might learn, particularly in the context of time-series data exhibiting underlying dynamical regimes, we employ the Koopman operator framework. We view the sequence of windows $\{x_t\}$ as states of a discrete-time dynamical system on $\cX$, evolving under the invariant measure $\mu$.

The Koopman operator provides a way to study the evolution of functions (observables) $\psi : \cX \xrightarrow{} \C$ defined on the state space. We focus on observables in the Hilbert space $L^2(\cX, \mu)$.

\begin{definition}[$\Delta$-Step Koopman Operator]
The $\Delta$-step Koopman operator $\cK \equiv \cK_{(\Delta)} : L^2(\cX, \mu) \xrightarrow{} L^2(\cX, \mu)$ transforms an observable $\psi$ into its conditional expectation $\Delta$ steps into the future:
\begin{equation}
(\cK\psi)(x) = \E[\psi(x_{t+\Delta}) \mid x_t = x].
\label{eq:koopman_def_framework}
\end{equation}
\end{definition}

A fundamental property of $\cK$ is its linearity, regardless of the non-linearity of the underlying system generating $x_t$. It is also a contraction on $L^2(\cX, \mu)$. Of particular interest are its eigenfunctions $\psi_j$, which satisfy $\cK \psi_j=\nu_j\psi_j$ for eigenvalues $\nu_j \in \C$. Eigenfunctions with $\nu_j=1$ represent quantities that are invariant in expectation under the $\Delta$-step dynamics. The components of JEPA's learned encoder $\fenc(x_t)$ can be seen as a vector of observables, and understanding their relationship with the Koopman operator is key to our analysis.

\section{Theoretical Justification for Clustering}
\label{sec:theory_main}
Having established the JEPA learning paradigm and the Koopman operator framework, we now develop our central theoretical argument. We will demonstrate that under specific assumptions about the data dynamics, \eqref{eq:jepa_loss_idealized_prelim_revised} incentivizes the encoder $f$ to learn representations that correspond to underlying, discrete dynamical regimes. This occurs because these regime-specific representations are characterized by functions that are invariant under the Koopman operator, making them optimally predictable.

\subsection{Dynamical Regimes and Invariant Observables}
\label{subsec:regimes_and_invariants}
The cornerstone of our argument is the assumption that the observed time-series data, while potentially complex, arises from a system that operates in a finite number of distinct dynamical modes or regimes.

\begin{assumption}[Finite Mixture of Ergodic Regimes] 
\label{ass:mixture_model_theory}
The invariant measure $\mu$ governing the time-series windows $x_t \in \cX$ decomposes as a finite convex mixture of $r$ distinct ergodic component measures: 
\begin{equation}
    \mu = \sum_{i=1}^r \alpha_i \mu_i, ~ \text{where} ~ r \in \N, ~ \alpha_i > 0,  \sum_{i=1}^r\alpha_i = 1
\end{equation}

Each $\mu_i$ is an ergodic invariant measure supported on a measurable set $\cX_i \subset \cX$. These supports $\{\cX_i\}_{i=1}^{r}$ are essentially disjoint, and trajectories starting in $\cX_i$ remain confined to $\cX_i$ for all future times (dynamical immiscibility).
\end{assumption}

This assumption allows us to define regime indicator functions $\chi_i(x) := \ones_{\cX_i}(x)$. These functions are linearly independent and span an $r$-dimensional subspace $\cV := \mathrm{span}\{\chi_1, \dots, \chi_r\}$.

\begin{lemma}[Properties of Regime Indicators and $\cV$] \label{lemma:V_properties_theory}
Under Assumption \ref{ass:mixture_model_theory}:
\begin{enumerate}[label=(\alph*)]
    \item Each regime indicator $\chi_i$ is an eigenfunction of $\cK$ with eigenvalue 1: $\cK \chi_i=\chi_i$
    \item Each $\chi_i$ is pathwise invariant over $\Delta$ steps: $\chi_i(x_{t+\Delta})=\chi_i(x_{t})$
    \item Consequently, any function $\psi \in \cV$ satisfies $\cK\psi=\psi$ and $\psi(x_{t+\Delta})=\psi(x_{t})$
    \item The subspace $\cV$ is precisely the eigenspace of $\cK$ corresponding to the eigenvalue $1$, and its dimension is $r$.
\end{enumerate}
\end{lemma}

\begin{proof}[Proof Intuition]
Properties (a) and (b) stem directly from the dynamical immiscibility of regimes: if a system is in regime $\cX_i$, it is expected to remain there, and its indicator $\chi_i$ will deterministically remain 1. Property (c) follows by linearity. Property (d) is a fundamental result from ergodic decomposition theory, linking the number of ergodic components to the dimension of the invariant subspace of the Koopman operator. The detailed proofs are provided in Appendix~\ref{app:proof_lemma_V_properties_detailed}.
\end{proof}

Lemma \ref{lemma:V_properties_theory}(c) is critical: functions in $\cV$ are not just invariant in expectation but are also perfectly predictable on a pathwise basis with respect to the regimes, i.e., their future value $f(x_{t+\Delta})$ is identical to their current value $f(x_{t})$. This makes them prime candidates for what JEPA might learn.

\subsection{JEPA Loss Minimization and Koopman Invariants}
\label{subsec:jepa_loss_koopman}
We now connect the idealized loss \eqref{eq:jepa_loss_idealized_prelim_revised} to the learning of these Koopman-invariant functions. For analytical clarity, especially in isolating the encoder's role in finding predictable structures, we introduce an assumption about the predictor.

\begin{assumption}[Linear Predictor] \label{ass:linear_predictor_theory}
The predictor $g: \R^k \to \R^k$ is a linear transformation, i.e., $g(z) = Mz$ for some matrix $M \in \R^{k \times k}$ with learnable parameters.
\end{assumption}

Under this assumption, let $f(x) = \vec{\psi}(x)$ be the vector of $k$ observables learned by the encoder. The idealized JEPA loss \eqref{eq:jepa_loss_idealized_prelim_revised} becomes $L(f, M) = \E_{x \sim \mu}\bigl[\|M \vec{\psi}(x) - \vec{\psi}(x_{t+\Delta})\|_2^2\bigr]$. This loss can be decomposed using the Koopman operator $\cK$, as derived in Appendix~\ref{app:loss_decomposition_theory_detailed}:
\begin{equation}
\begin{split}
L(f, M) &= \underbrace{\E_x \bigl[ \|M \vec{\psi}(x) - (\cK\vec{\psi})(x)\|_2^2 \bigr]}_{\text{Term 1: Mean Prediction Error}} \\
&\quad + \underbrace{\E_x \bigl[ \|(\cK\vec{\psi})(x) - \vec{\psi}(x_{t+\Delta})\|_2^2 \bigr]}_{\text{Term 2: Inherent Stochasticity Error}}
\end{split}
\label{eq:loss_decomp_koopman_theory}
\end{equation}

JEPA aims to minimize $L(f,M)$ by jointly optimizing $f$ (i.e., $\vec{\psi}$) and $M$. The loss is zero if and only if both Term 1 and Term 2 are zero. Term 2 quantifies the degree to which the learned observables $\vec{\psi}(x)$ deviate from their conditional expectation $(\cK\vec{\psi})(x)$ along actual trajectories. Term 1 quantifies how well the linear predictor $M \vec{\psi}(x)$ can match this conditional expectation.

Our theory shows that functions spanning the regime-invariant subspace $\cV$ are optimal for minimizing this loss.

\begin{theorem}[JEPA Learns Regime Indicators] \label{thm:jepa_learns_V_theory}
Let Assumptions \ref{ass:mixture_model_theory} and \ref{ass:linear_predictor_theory} hold. Assume the encoder has sufficient capacity, i.e., its latent dimension $k \ge r$.
The JEPA loss \eqref{eq:loss_decomp_koopman_theory} achieves its global minimum if and only if:
\begin{enumerate}[label=(\alph*)]
    \item The components $f_j$ of the encoder output are such that $(\cK f_j)(x)=f_j(x_{t+\Delta})$ for $\mu$ almost everywhere (a.e.).
    \item The predictor matrix $M$ satisfies $Mf(x)=(\cK f)(x)$ for $\mu$ a.e. $x$.
\end{enumerate}

These conditions are simultaneously satisfied if the components $f_j(x)$ of the encoder $f(x)$ belong to the invariant subspace $\cV$, and the predictor matrix $M$ acts as the identity transformation on the subspace of $\R^k$ spanned by $f(\cX)$.

Specifically, if $f^*(x)=(\chi_1(x), \dots, \chi_r(x), \vec{0}_{k-r})^{T}$, then $L(f^*,M)$ is minimized by any $M^*$ whose action on the subspace spanned by the non-zero components of $f^*(\cX)$ is identity and zero elsewhere.
\end{theorem}

\begin{proof}[Proof Intuition]
If each component $f_j \in \cV$, then by Lemma \ref{lemma:V_properties_theory}(c), we have both $(\cK f_j)(x)=f_j(x)$ and $f_j(x_{t+\Delta})=f_j(x_{t})$.

Thus, if the encoder learns functions $f_j$ that are (or span) the regime indicators in $\cV$, both terms of the loss can be driven to zero with a predictor $M$ that effectively acts as an identity map on these learned, invariant representations. 

For the specific $f^*$ given, $M^*$ being identity on the first $r$ components achieves this. Conversely, for the loss to be zero, Term 2 requires $(\cK f_j)(x_t)=f_j(x_{t+\Delta})$ almost surely (a.s.). If $M$ then ensures Term 1 is zero by $Mf_j=\cK f_j$, and if $M$ is to be simple (e.g., identity-like for these $f_j$), it implies $\cK f_j=f_j$, pushing $f_j$ towards $\cV$. The full details are provided in Appendix~\ref{app:proof_thm_jepa_learns_V_detailed}.
\end{proof}

\subsection{Implication for Latent Space Clustering}
\label{subsec:clustering_implication}
Theorem \ref{thm:jepa_learns_V_theory} provides a direct explanation for the empirically observed clustering. If the JEPA encoder $f(x)$, with latent dimension $k \ge r$, learns representations whose components $f_j$ span the $r$-dimensional invariant subspace $\cV$, then $f(x)$ effectively becomes a representation of the regime indicators.

For instance, if $f(x)$ is an invertible linear transformation of the vector $\vec{\chi}(x) = (\chi_1(x), \dots, \chi_r(x))^T$, say $f(x) = A \vec{\chi}(x)$ for a $k \times r$ matrix $A$ of rank $r$. When a window $x$ belongs to regime $\cX_i$, $\vec{\chi}(x)$ becomes $e_i$ (the $i$-th standard basis vector in $\R^r$). Consequently, the encoder output is $f(x) = A e_i$, which is simply the $i$-th column of matrix $A$.

Therefore, all input windows $x_t$ originating from the same dynamical regime $\cX_i$ are mapped by the encoder $f$ to the same space (or a very tight region, allowing for approximation errors and noise) $A e_i$ in the latent space $R^k$. Since $A$ has rank $r$, the $r$ vectors $\{A e_i\}_{i=1}^r$ are distinct (or linearly independent if $k \ge r$). This mapping naturally results in $r$ distinct clusters in the latent space, with each cluster corresponding precisely to one of the underlying regimes.

While our core theoretical result relies on the assumption of a linear predictor $g(z)=Mz$ for analytical tractability, the general intuition extends to non-linear predictors. A sufficiently expressive non-linear predictor $\fpred$ would aim to approximate the conditional expectation $\fpred(f(x_t)) \approx \E[f(x_{t+\Delta}) \mid f(x_t)]$. If the encoder $f$ learns representations $z_t = f(x_t)$ that are elements of $\cV$, then $z_{t+\Delta} = z_t$ a.s. In this scenario, the conditional expectation $\E[z_{t+\Delta} \mid z_t] = z_t$. Thus, an optimal non-linear predictor would learn to approximate an identity map, $\fpred(z_t) \approx z_t$, for these highly predictable, regime-specific representations. The fundamental drive to find representations $f(x_t)$ for which $f(x_{t+\Delta})$ is ``simply" predictable from $f(x_t)$ remains, and functions in $\cV$ (for which $f(x_{t+\Delta})=f(x_t)$) represent the epitome of such simplicity.

\section{Experimental Validation}
\label{sec:experiments}
To rigorously test the predictions of our Koopman-based theory, we employ a novel empirical strategy. Using a synthetic dataset with known underlying regimes, we perform a series of targeted, quantitative analyses on the learned linear predictor, $M$, to verify its predicted behavior as an identity operator on the learned invariant subspace.

\subsection{Synthetic Dataset Generation}
\label{subsec:dataset_generation}
To create an environment where Assumption \ref{ass:mixture_model_theory} (Finite Mixture of Ergodic Regimes) is satisfied by construction, we generate a synthetic dataset comprising $r=18$ distinct dynamical regimes. The objective is to provide the JEPA model with input data that clearly embodies the structured, multi-modal dynamics our theory addresses. Each regime is designed to exhibit unique temporal characteristics, ensuring clear distinctions and facilitating the analysis of JEPA's ability to differentiate them.

Master sequences, each of length $L_{master} = 1024$ time steps, are generated for every regime. For these experiments, additive observation noise is set to zero for all deterministic signal components, ensuring the cleanest possible testbed to verify our theoretical claims without the confounding effects of stochastic observation noise. Stochastic processes, such as Autoregressive Moving Average (ARMA) models, naturally incorporate their own intrinsic process noise.

The repertoire of $r=18$ regimes encompasses a diverse set of dynamics crucial for testing the robustness of our theory: Periodic signals include several sinusoidal variations differing in frequency, amplitude, and harmonics, with phases randomized per sequence. Square waves and a sawtooth wave provide examples of non-smooth periodicities. 

To model stochastic dynamics, we include Autoregressive (AR) models with varying dependency coefficients, a Moving Average (MA) process, and a mixed ARMA model.

Aperiodic and event-based signals are represented by linear trends with both positive and negative slopes, featuring per-sequence randomization of slope and intercept, and sequences with sparse, randomly located positive pulses. Finally, to explore more complex interactions, the dataset includes combined signals like a sinusoid superimposed on a linear trend, and a high-noise variant consisting of a sinusoid with significantly increased internal process noise. Examples of each type of waveform are depicted in Figure \ref{fig:dataset_waveforms}, and a comprehensive list of parameters for each generative process is detailed in Appendix \ref{app:regime_params}.

\begin{figure}[t]
\centering
\includegraphics[width=\linewidth]{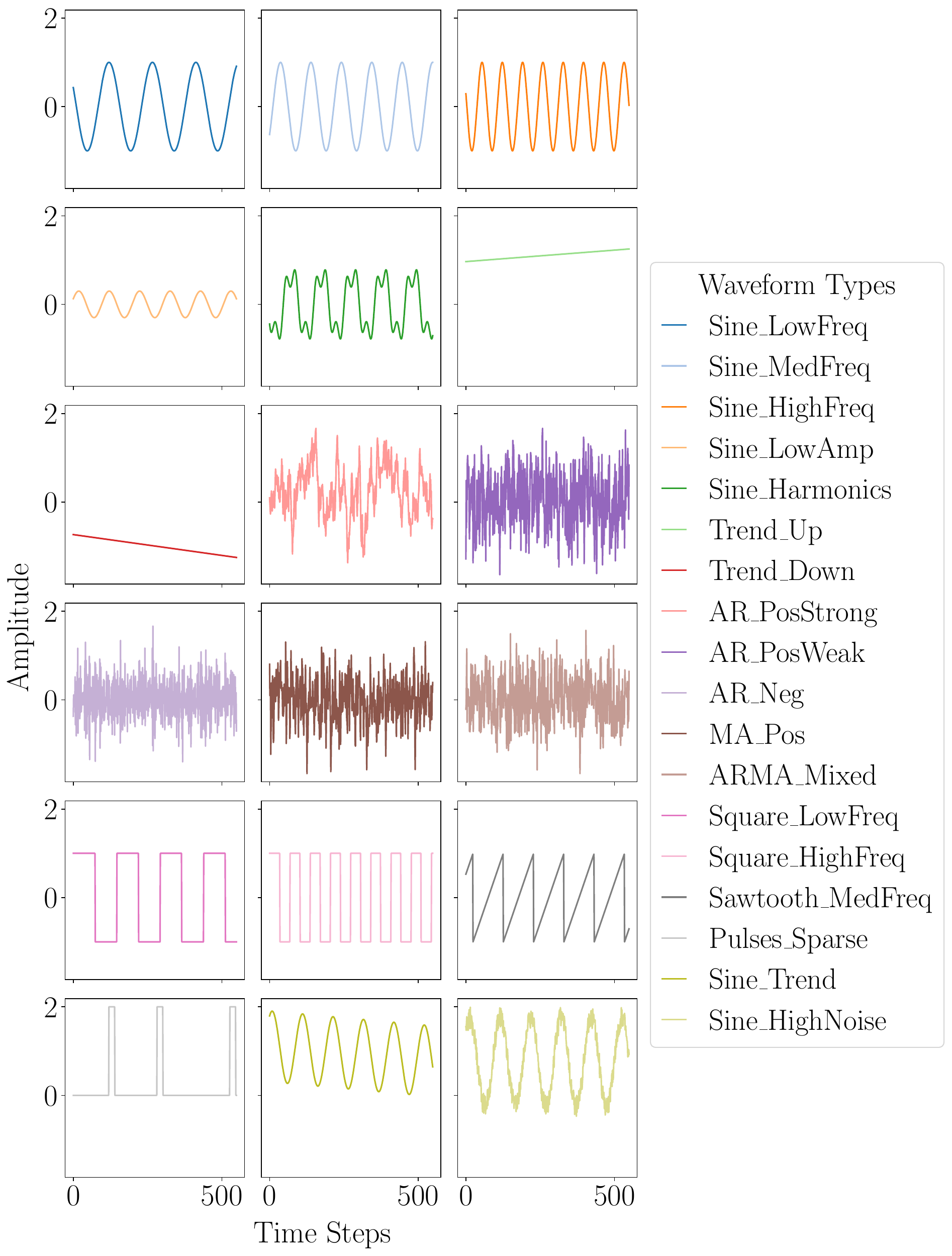}
\caption{Example waveforms illustrating the diversity of the synthetic dataset. }
\label{fig:dataset_waveforms}
\end{figure}

From each master sequence, we extract a single context-target pair.  The context window, $x_{t}$, consists of the initial $n_c = 768$ steps. The target window, $x_{t+\Delta}$, is defined by shifting forward by a prediction horizon of $\Delta=256$ steps, also taking a 768-step window.

This windowing scheme creates a substantial overlap: the steps from $s_{\Delta}$ to $s_{n_c-1}$ are present in both the context and target. This design serves a dual purpose: it tasks the model with maintaining representational consistency for the overlapping data, while also requiring it to predict the representation of the novel future segment ($s_{n_c}$ to $s_{\Delta+n_c-1}$). We note that experiments with a completely non-overlapping windowing scheme ($\Delta \ge n_c$) also yielded satisfactory results.

A total of 10,000 sequences are generated for each of the $r=18$ regimes, yielding 180,000 context-target pairs. This dataset is deterministically partitioned at the sequence level into training (70\%), validation (20\%), and test (10\%) sets, with all regime types proportionally represented across splits. Prior to window extraction, each master sequence undergoes per-sequence standardization to encourage the model to learn shape-based and relative dynamical features.

\subsection{JEPA Model Configuration}
\begin{figure*}[htbp]
    \centering
    \includegraphics[width=\textwidth]{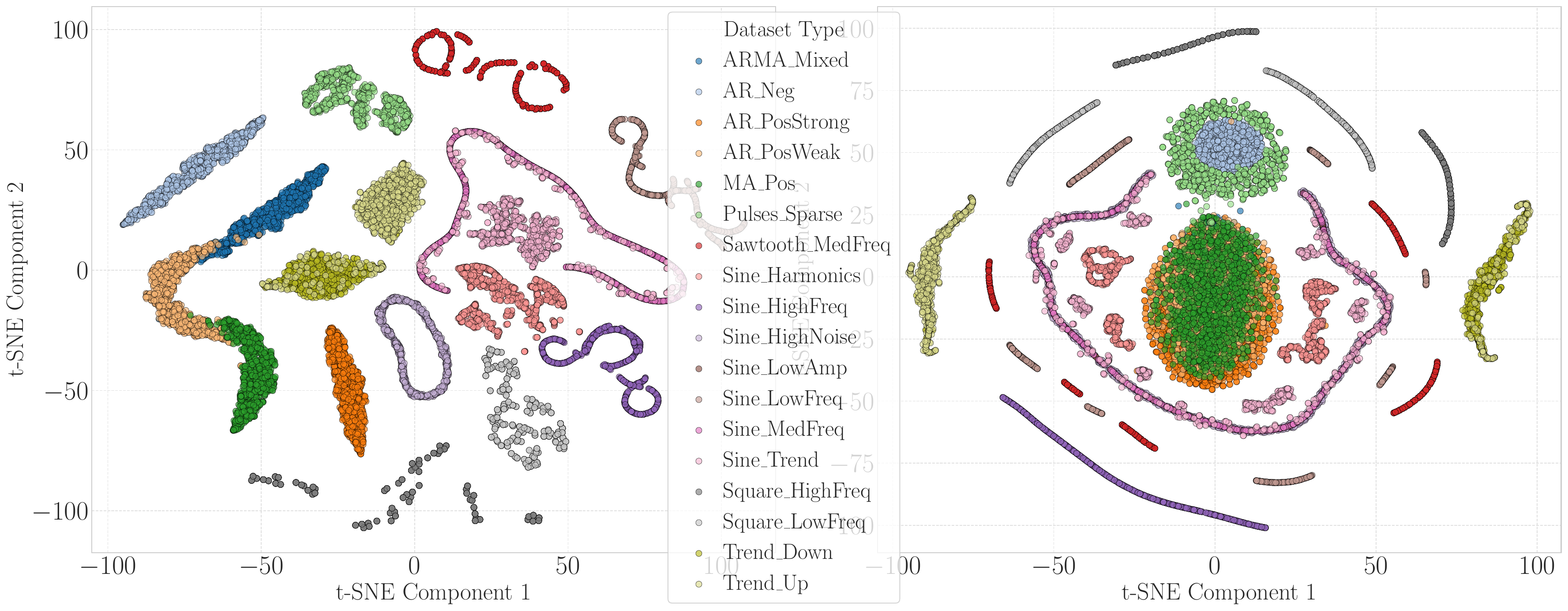}
    \caption{Latent space visualization (t-SNE) of test set embeddings. Left: JEPA embeddings often form distinct clusters corresponding to underlying dynamical regimes (indicated by colors). Right: Embeddings from a conventional autoencoder with an identical encoder architecture may not exhibit such clear regime-based separation on the same data.}
    \label{fig:actual_jepa_vs_ae_clustering}
\end{figure*}

\label{subsec:model_config}
To instantiate the JEPA model whose idealized behavior was analyzed in Section \ref{sec:theory_main}, and to empirically test our theoretical predictions, we configure the encoder and predictor architectures with specific considerations for this study. All models are implemented in PyTorch.

The encoder $\fenc$ is a one-dimensional Convolutional Neural Network designed to process the input time-series windows $x_t$. Its architecture consists of four convolutional layers, each followed by ReLU activation. The output of the convolutional blocks is flattened and passed through a linear projection head to produce a latent representation $z_t \in \R^k$. The latent dimension $k$ was set to 32, ensuring $k \ge r$, providing sufficient capacity for the encoder to potentially represent the $r$ distinct regimes as predicted by our theory. 

A key component for directly testing the predictions regarding the predictor's behavior (Theorem \ref{thm:jepa_learns_V_theory}) involves using a linear predictor for $\fpred$. This predictor implements a linear transformation $g(z)=Mz$, where $M \in \R^{k\times k}$, and was configured without bias to simplify analysis. To specifically probe for the existence and stability of the theoretically optimal identity-like solution, our primary analysis stems from experiments where $M$ was initialized as an identity matrix. To explore the broader optimization landscape under more standard conditions, separate control experiments were conducted where $M$ was initialized using a standard random scheme. This dual approach allows us to both verify the existence of an interpretable solution and assess its uniqueness.

For comparison, and to observe clustering under more typical JEPA conditions, we also configure experiments using a standard non-linear multi-layer perceptron (MLP) predictor. This predictor consists of two hidden layers with ReLU activations, where the hidden dimension is two times the latent dimension $k$. The full details of the model are provided in Appendix~\ref{app:model_architectures}.

The target encoder $\fEMA$ is a direct copy of the online encoder $\fenc$, with its parameters updated via an Exponential Moving Average (EMA) using a decay rate of $\alpha=0.996$.

\section{Results and Discussion}
\label{sec:results}
Our analysis focuses on verifying the key predictions derived from our Koopman operator-based theory: namely, the emergence of regime-aligned clustering in the latent space and the characteristic behavior of the learned predictor, particularly in the idealized linear case. All reported results are from evaluations on the held-out test set.

A primary outcome of our theory is that JEPA's encoder $\fenc$ should learn representations $z_t=\fenc(x_t)$ that distinguish the underlying dynamical regimes, leading to distinct clusters in the latent space. To investigate this, we projected the $k$-dimensional latent embeddings of test set windows into a 2D space using t-SNE \cite{vanderMaaten2008}. The points were then color-coded according to their ground-truth regime labels.

Figure \ref{fig:actual_jepa_vs_ae_clustering} presents these visualizations. As hypothesized, the JEPA model trained with a standard non-linear MLP predictor demonstrates a clear formation of clusters that align strongly with the ground-truth dynamical regimes. We quantified this alignment using K-Means clustering ($\text{K}=18$), which revealed a mean cluster purity of $65.48\%$ for JEPA's embeddings. This emergent order supports our argument that the predictive objective learns features separable by the underlying data-generating modes.

In contrast, a conventional autoencoder with an identical encoder architecture fails to separate the dynamical regimes for the same data, achieving a mean cluster purity of only $38.81\%$. This difference highlights the efficacy of JEPA's abstract predictive objective over a purely reconstructive one for uncovering and organizing representations by their underlying dynamical structure.

\subsection{Analysis of the Learned Linear Predictor Matrix}
\label{subsec:results_matrix_M}

Our theory (Theorem \ref{thm:jepa_learns_V_theory}), under the assumption of a linear predictor $g(z)=Mz$, implies that if the encoder learns regime indicator functions (elements of $\cV$), then $M$ should act as an identity transformation on the subspace spanned by these learned representations. We investigate this by analyzing the $k \times k$ matrix $M$ learned by the JEPA model equipped with a linear predictor (initialized as $M \approx I_k$).

The learned matrix $M$ from our identity-initialized experiment converged to a near-perfect identity transformation. This was confirmed through three key quantitative properties. First, its deviation from the identity matrix $I_k$ was minimal, with a relative Frobenius error ($||M - I_k||_F / ||M||_F$) of just $2.34\%$. Second, the matrix was highly symmetric, another crucial property of an identity operator, with its skew-symmetric relative norm ($||M - M^T||_F / ||M||_F$) measuring only $2.06\%$. Finally, and most critically, its eigenvalue spectrum, displayed in Figure~\ref{fig:M_heatmap_eigen}, reveals the mechanism behind this behavior: it is dominated by $r$ eigenvalues near 1.0, indicating that $M$ has learned to preserve the $r$-dimensional subspace of the learned regimes while potentially attenuating all other, less predictable dimensions.

\begin{figure}[tbp]
  \centering
  \includegraphics[width=\linewidth]{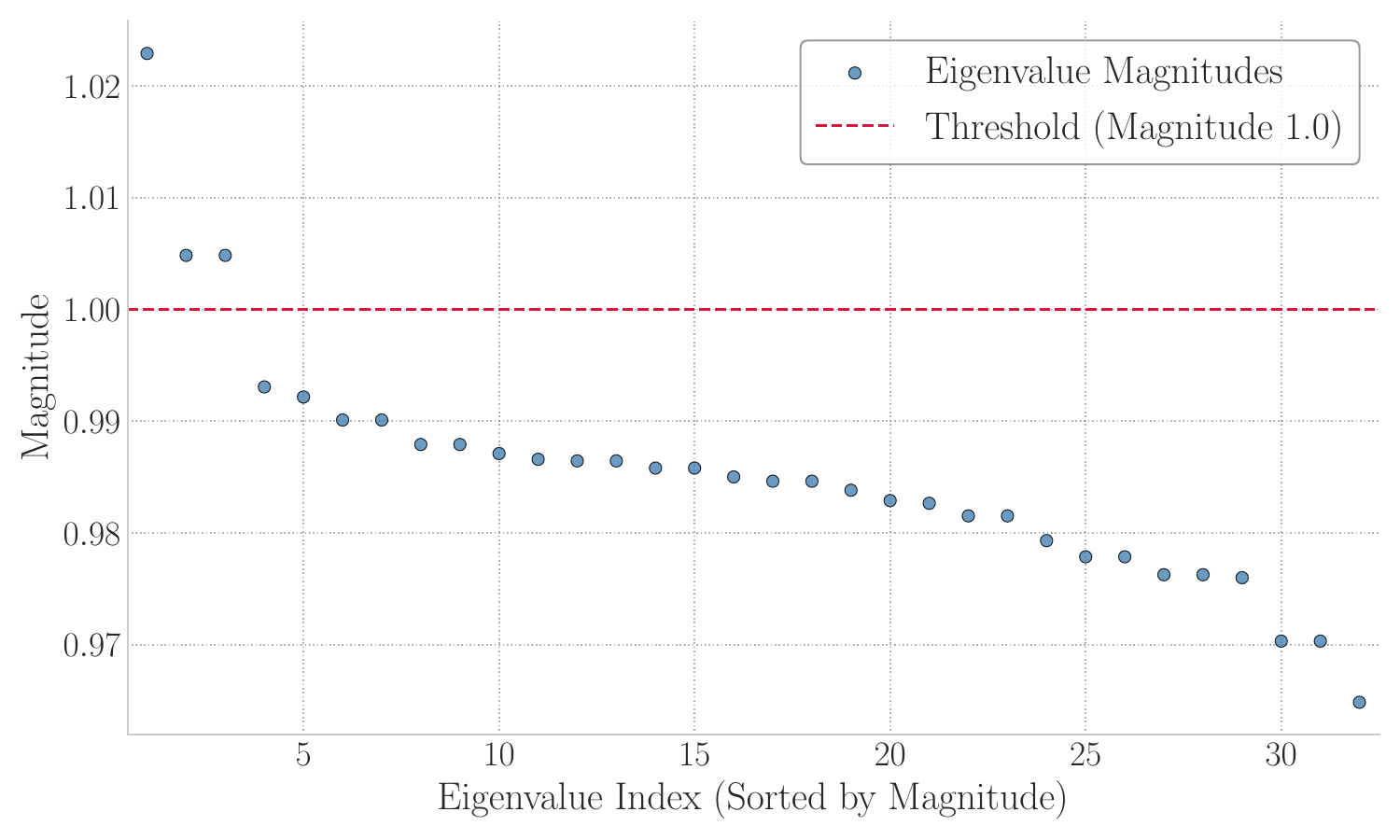}
  \caption{Sorted magnitudes of the eigenvalues of $M$, showing dominant eigenvalues near 1.0.}
  \label{fig:M_heatmap_eigen}
\end{figure}

To further test $M$ as an identity operator on the representations of the learned regimes, we identified $r=18$ cluster centroids $\{c_i\}_{i=1}^{18}$ from the test set embeddings using K-Means. We then computed the relative Euclidean norm of the difference $\|Mc_i-c_i\|_2/\|c_i\|_2$ for each centroid. The error was consistently small, averaging just $0.80\%$ across all centroids, as shown in Figure~\ref{fig:M_centroid_action}. This confirms that $M$ preserves the locations of the regime centroids, aligning with the theoretical prediction that $Mf(x)\approx f(x)$ for $f(x) \in \cV$.

Crucially, this interpretable solution is not unique. Our control experiments with a randomly initialized predictor $M$ converged to the same low loss but yielded a dense, non-identity transformation, while also exhibiting clear visual clustering. This finding does not contradict our theory but instead highlights its core implication: the JEPA loss is invariant to any invertible linear transformation (a change of basis) applied to the latent space. While a random initialization finds an equally valid but entangled basis for the optimal subspace, our identity-initialized experiment proves that a canonical, interpretable solution exists and is a stable optimum. This demonstrates that while JEPA's objective successfully identifies the correct invariant subspace, an inductive bias, such as initializing the predictor toward identity, guides the model to a human-interpretable representation.

\begin{figure}[tbp]
  \centering
  \includegraphics[width=\linewidth]{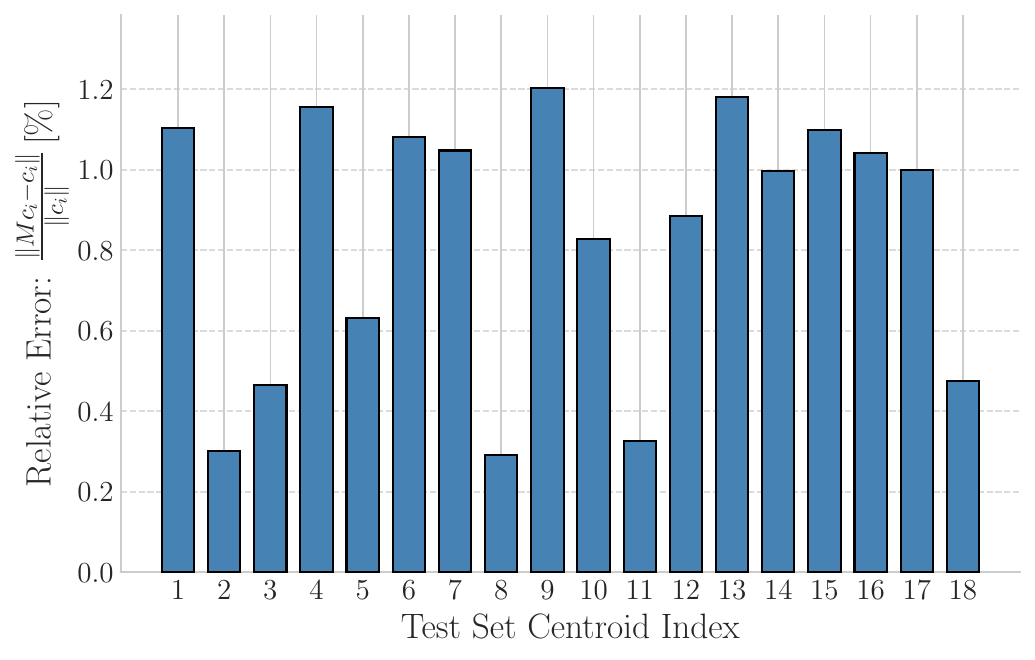}
  \caption{Action of the learned linear predictor $M$ on regime cluster centroids $c_i$. The small relative error indicates that $M$ largely preserves these regime-representative vectors.}
  \label{fig:M_centroid_action}
\end{figure}

The empirical results strongly align with the theoretical framework from Section \ref{sec:theory_main}. JEPA encoders form latent space clusters that directly correspond to the $r$ ground-truth dynamical regimes (Figure \ref{fig:actual_jepa_vs_ae_clustering}), a stark contrast to a standard autoencoder with an identical architecture. This demonstrates that JEPA's predictive objective, not just its capacity, is responsible for uncovering the underlying dynamical structure.

The analysis of the linear predictor $M$ further supports our Koopman-based theory. Its dominant eigenvalues near 1.0 and its near-identity action on cluster centroids both confirm that the predictor learns to preserve the representations of the stable regimes. This is precisely the behavior expected if the encoder $\fenc$ successfully learns functions spanning the invariant subspace $\cV$, whose elements $\psi\in\cV$ satisfy both $\cK\psi=\psi$ and $\psi(x_{t+\Delta})=\psi(x_{t})$ (Lemma \ref{lemma:V_properties_theory}).

Several limitations of the current study should be acknowledged. Our empirical validation relies on synthetic data with well-defined, immiscible regimes, which perfectly fit the finite mixture of ergodic regimes model of Assumption~\ref{ass:mixture_model_theory}. Real-world time series, however, often feature more complex phenomena like gradual transitions or hierarchical structures, and investigating JEPA's behavior under these conditions is an important next step.

On the theoretical side, our analysis made two key idealizations. We utilized a linear predictor for analytical tractability and approximated the EMA dynamics with $\fEMA\approx\fenc$. While our results are strong under these conditions, a more formal treatment of non-linear predictors and a deeper analysis of the EMA's stabilizing role would provide a more complete understanding of its role in stabilizing the learning of these invariants.

\section{Conclusion}
\label{sec:conclusion}

This paper has addressed why Joint-Embedding Predictive Architectures often exhibit emergent clustering of their latent representations according to underlying dynamical regimes in time-series data. We proposed a novel theoretical explanation rooted in Koopman operator theory, hypothesizing that JEPA's objective incentivizes its encoder to learn regime indicator functions. These indicators are invariant eigenfunctions (eigenvalue 1) of the $\Delta$-step Koopman operator and characterize the distinct, dynamically immiscible modes of behavior assumed to be present in the data.

Our theoretical derivation, under idealized conditions including a finite mixture of ergodic regimes and a linear predictor, proves that the JEPA loss is minimized when the encoder learns to span this space of indicators. This provides a first-principles basis for the observed clustering, where inputs from the same regime map to common latent locations. Our empirical results on synthetic data strongly support these predictions, showing clear clustering and the predicted behavior of the matrix $M$.

This work contributes to a deeper principled understanding of JEPA mechanisms grounded in dynamical systems theory. These insights can inspire new SSL objectives that explicitly leverage Koopman theory for more robust and interpretable representation learning. This could lead to improved unsupervised tools for system decomposition, regime identification, and anomaly detection, where deviations from learned regime clusters would signify anomalies.

Future work includes validating these findings on real-world datasets, comparing JEPA's implicit discovery to explicit identification methods, and exploring whether the architecture can learn other Koopman eigenfunctions corresponding to different dynamical modes, such as oscillations or slow decay.

\section*{Acknowledgments}
Funded by the European Union (grant no. 101168880). Views and opinions expressed are however those of the author(s) only and do not necessarily reflect those of the European Union. The European Union can not be held responsible for them. Project website: \url{https://dn-isense.eu/}.

For the purpose of open access (OA), as required by Horizon Europe (HE), the author has applied a CC BY public copyright license to the Author Accepted Manuscript (AAM) version resulting from this submission. 

\bibliography{aaai2026}
\appendix

\section{Detailed Proofs for Section~\ref{sec:theory_main}}
\label{sec:appendix_proofs_detailed}

This appendix provides the detailed mathematical derivations for the lemmas, theorems, and loss decomposition presented in Section~\ref{sec:theory_main} of the main text. All notation and definitions established therein are adopted here.

\subsection{Proof of Lemma~\ref{lemma:V_properties_theory} (Properties of Regime Indicators and $\cV$)}
\label{app:proof_lemma_V_properties_detailed}

\begin{lemma}[\ref{lemma:V_properties_theory} restated]
Under Assumption \ref{ass:mixture_model_theory} (Finite Mixture of Ergodic Regimes):
\begin{enumerate}[label=(\alph*)]
    \item Each regime indicator $\chi_i$ is an eigenfunction of the $\Delta$-step Koopman operator $\cK$ with eigenvalue 1: $\cK\chi_i = \chi_i$ ($\mu$-a.e.).
    \item Each $\chi_i$ is pathwise invariant over $\Delta$ steps: $\chi_i(x_{t+\Delta}) = \chi_i(x)$ ($\mu$-a.e.).
    \item Consequently, any function $\psi \in \mathcal{V} = \mathrm{span}\{\chi_1, \dots, \chi_r\}$ satisfies $\cK\psi = \psi$ and $\psi(x_{t+\Delta}) = \psi(x)$ ($\mu$-a.e.).
    \item The subspace $\cV$ is precisely the eigenspace of $\cK$ corresponding to the eigenvalue 1, and its dimension is $r$.
\end{enumerate}
\end{lemma}
\begin{proof}

(a) Let $x \in \cX$. The $\Delta$-step Koopman operator is defined as $(\cK\chi_i)(x) = \E[\chi_i(x_{t+\Delta}) \mid x_t = x]$.
\begin{itemize}
    \item If $x \in \cX_i$: By Assumption~\ref{ass:mixture_model_theory} (dynamical immiscibility), if $x_t = x \in \cX_i$, then $x_{t+\Delta} \in \cX_i$ a.s. Thus, $\chi_i(x_{t+\Delta}) = 1$ a.s. It follows that $(\cK\chi_i)(x) = \E[1 \mid x_t=x] = 1$. Since $x \in \cX_i$, $\chi_i(x) = 1$, so $(\cK\chi_i)(x) = \chi_i(x)$.
    \item If $x \notin \cX_i$: This implies $x \in \cX_j$ for some $j \neq i$, as the supports $\{\cX_l\}_{l=1}^r$ form an essential partition of $\cX$ under $\mu$. By dynamical immiscibility, if $x_t = x \in \cX_j$ ($j \neq i$), then $x_{t+\Delta} \in \cX_j$ a.s., meaning $x_{t+\Delta} \notin \cX_i$ a.s. Thus, $\chi_i(x_{t+\Delta}) = 0$ a.s. It follows that $(\cK\chi_i)(x) = \E[0 \mid x_t=x] = 0$. Since $x \notin \cX_i$, $\chi_i(x) = 0$, so $(\cK\chi_i)(x) = \chi_i(x)$.
    \end{itemize}
Since these cases cover $\cX$ up to a set of $\mu$-measure zero, $\cK\chi_i = \chi_i$ holds $\mu$-a.e.

(b) This follows directly from the same arguments as in (a) concerning dynamical immiscibility:
\begin{itemize}
    \item If $x \in \cX_i$: Then $\chi_i(x)=1$. Since $x_{t+\Delta} \in \cX_i$ a.s., $\chi_i(x_{t+\Delta})=1$ a.s. Thus $\chi_i(x_{t+\Delta})=\chi_i(x)$ a.s. for $x \in \cX_i$.
    \item If $x \notin \cX_i$: Then $\chi_i(x)=0$. Since $x_{t+\Delta} \notin \cX_i$ a.s., $\chi_i(x_{t+\Delta})=0$ a.s. Thus $\chi_i(x_{t+\Delta})=\chi_i(x)$ a.s. for $x \notin \cX_i$.
\end{itemize}
Therefore, $\chi_i(x_{t+\Delta})=\chi_i(x_t)$ holds $\mu$-a.e.

(c) Let $\psi \in \mathcal{V}$. Then $\psi(x) = \sum_{j=1}^r c_j \chi_j(x)$ for some constants $c_j \in \R$. By linearity of $\cK$ and part (a):
\begin{align*}
    (\cK\psi)(x) &= \cK\left(\sum_{j=1}^r c_j \chi_j\right)(x) = \sum_{j=1}^r c_j (\cK\chi_j)(x) \\&= \sum_{j=1}^r c_j \chi_j(x) = \psi(x) \quad (\mu\text{-a.e.})
\end{align*}
Similarly, using part (b):
\begin{align*}
    \psi(x_{t+\Delta}) &= \sum_{j=1}^r c_j \chi_j(x_{t+\Delta}) \\&= \sum_{j=1}^r c_j \chi_j(x) = \psi(x) \quad (\mu\text{-a.e.})
\end{align*}

(d) The functions $\{\chi_1, \dots, \chi_r\}$ are linearly independent in $L^2(\mu)$ because their supports $\cX_i$ are essentially disjoint and $\mu(\cX_i) = \alpha_i > 0$. From part (a), they are all eigenfunctions of $\cK$ with eigenvalue 1. Thus, the eigenspace corresponding to eigenvalue 1 has dimension at least $r$. A fundamental result in ergodic theory states that for a measure-preserving transformation, the dimension of the eigenspace of its Koopman operator corresponding to eigenvalue 1 is equal to the number of ergodic components of the invariant measure \cite{Petersen1983ErgodicT, Eisner2015OperatorTE}. Under Assumption \ref{ass:mixture_model_theory}, there are $r$ such ergodic components $(\cX_i, \mu_i)$. Therefore, the space of functions in $L^2(\mu)$ that are invariant under $\cK$ has dimension $r$.
\end{proof}

\subsection{Derivation of JEPA Loss Decomposition \eqref{eq:loss_decomp_koopman_theory}}
\label{app:loss_decomposition_theory_detailed}
Let $f(x) = \vec{\psi}(x)$ be the $k$-dimensional output of the encoder. Under Assumption \ref{ass:linear_predictor_theory}, the predictor is $g(z) = Mz$. The idealized JEPA loss is $L(f, M) = \E_{x \sim \mu}\bigl[\|M \vec{\psi}(x) - \vec{\psi}(x_{t+\Delta})\|_2^2\bigr]$.

Let $A(x) = M \vec{\psi}(x) - (\cK\vec{\psi})(x)$ and $B(x) = (\cK\vec{\psi})(x) - \vec{\psi}(x_{t+\Delta})$.

Then $M \vec{\psi}(x) - \vec{\psi}(x_{t+\Delta}) = A(x) + B(x)$.

The term inside the expectation in the loss is $\|A(x) + B(x)\|_2^2$.
\begin{align*}
    \|A(x) + B(x)\|_2^2 &= \|A(x)\|_2^2 + \|B(x)\|_2^2 + 2 \langle A(x), B(x) \rangle_2
\end{align*}

Taking the expectation $\E_x \equiv \E_{x \sim \mu}$:
\begin{align*}
    L(f, M) &= \E_x \left[ \|A(x)\|_2^2 \right] + \E_x \left[ \|B(x)\|_2^2 \right] \\&+ 2 \E_x \left[ \langle A(x), B(x) \rangle_2 \right]
\end{align*}

We analyze the cross term $\E_x \left[ \langle A(x), B(x) \rangle_2 \right]$. Using the law of total expectation, $\E_x[\cdot] = \E_{x_t \sim \mu} \left[ \E[\cdot \mid x_t] \right]$.
\begin{align*}
    \E_x \left[ \langle A(x), B(x) \rangle_2 \right] 
    &= \E_{x_t \sim \mu} \left[ \E \left[ \langle M \vec{\psi}(x_t) - (\cK\vec{\psi})(x_t), \right. \right. \\
    &\qquad \left. \left. (\cK\vec{\psi})(x_t) - \vec{\psi}(x_{t+\Delta}) \rangle_2 \mid x_t \right] \right]
\end{align*}

Given $x_t$, the term $A(x_t) = M \vec{\psi}(x_t) - (\cK\vec{\psi})(x_t)$ is fixed, as $\vec{\psi}(x_t)$ is known and $(\cK\vec{\psi})(x_t)$ is a deterministic function of $x_t$.
So, the inner conditional expectation becomes:
\begin{align*}
    &\quad \left\langle M \vec{\psi}(x_t) - (\cK\vec{\psi})(x_t), \E \left[ (\cK\vec{\psi})(x_t) - \vec{\psi}(x_{t+\Delta}) \mid x_t \right] \right\rangle_2 \\
    &= \left\langle M \vec{\psi}(x_t) - (\cK\vec{\psi})(x_t), (\cK\vec{\psi})(x_t) - \E [ \vec{\psi}(x_{t+\Delta}) \mid x_t ] \right\rangle_2
\end{align*}

By definition of the Koopman operator \eqref{eq:koopman_def_framework} applied element-wise, $\E [ \vec{\psi}(x_{t+\Delta}) \mid x_t ] = (\cK\vec{\psi})(x_t)$.

Thus, the second argument of the inner product is $(\cK\vec{\psi})(x_t) - (\cK\vec{\psi})(x_t) = \vec{0}$.

The inner product is therefore 0 for any $x_t$. Consequently, its expectation $\E_x \left[ \langle A(x), B(x) \rangle_2 \right] = 0$.

This leads to the decomposition:
\begin{equation*}
\begin{split}
L(f, M) &= \E_x \bigl[ \|M \vec{\psi}(x) - (\cK\vec{\psi})(x)\|_2^2 \bigr] \\
& + \E_x \bigl[ \|(\cK\vec{\psi})(x) - \vec{\psi}(x_{t+\Delta})\|_2^2 \bigr],
\end{split}
\end{equation*}
which is \eqref{eq:loss_decomp_koopman_theory}.

\subsection{Proof of Theorem~\ref{thm:jepa_learns_V_theory} (JEPA Learns Regime Indicator Functions)}
\label{app:proof_thm_jepa_learns_V_detailed}
\begin{theorem}[\ref{thm:jepa_learns_V_theory} restated]
Let Assumptions \ref{ass:mixture_model_theory} and \ref{ass:linear_predictor_theory} hold. Assume encoder capacity $k \ge r$.
The JEPA loss \eqref{eq:loss_decomp_koopman_theory} achieves its global minimum of 0 if and only if:
\begin{enumerate}[label=(\alph*)]
    \item $(\mathcal{K}f_j)(x) = f_j(x_{t+\Delta})$ for $\mu$-a.e. $x$ and for each component $j=1,\dots,k$. (Term 2 is zero).
    \item $M f(x) = (\mathcal{K}f)(x)$ for $\mu$-a.e. $x$. (Term 1 is zero).
\end{enumerate}

These conditions are simultaneously satisfied if the components $f_j(x)$ of $f(x)$ are chosen from $\mathcal{V}$, and $M$ acts as identity on the subspace of $\R^k$ spanned by $f(\mathcal{X})$.

Specifically, if $f^*(x)=(\chi_1(x), \dots, \chi_r(x), \vec{0}_{k-r})^{T}$, then $L(f^*,M)$ is minimized by any $M*$ whose action on the subspace spanned by the non-zero components of $f^*(\cX)$ is identity and zero elsewhere.
\end{theorem}
\begin{proof}
The loss $L(f,M)$ is a sum of two non-negative terms (squared norms). Thus, $L(f,M)=0$ if and only if both terms are individually zero $\mu$-a.e. This directly gives conditions (a) and (b) of the theorem statement.

Now, we show sufficiency: if components $f_j \in \mathcal{V}$ and $M$ acts as identity on $f(\mathcal{X})$.
If each component $f_j \in \mathcal{V}$, then by Lemma~\ref{lemma:V_properties_theory}(c):
\begin{enumerate}[label=\arabic*.]
    \item $(\mathcal{K}f_j)(x) = f_j(x)$ ($\mu$-a.e.)
    \item $f_j(x_{t+\Delta}) = f_j(x)$ ($\mu$-a.e.)
\end{enumerate}
Substituting these into Term 2 of \eqref{eq:loss_decomp_koopman_theory} for each component:
$$ \|(\mathcal{K}f_j)(x) - f_j(x_{t+\Delta})\|_2^2 = \|f_j(x) - f_j(x)\|_2^2 = 0 \quad (\mu\text{-a.e.}), $$
so Term 2 is zero if $f_j \in \mathcal{V}$ for all $j$.

Substituting $(\mathcal{K}f_j)(x) = f_j(x)$ into Term 1:
$$ \|(Mf)(x)_j - (\mathcal{K}f_j)(x)\|_2^2 = \|(Mf)(x)_j - f_j(x)\|_2^2. $$

If $M$ is chosen such that it acts as the identity transformation on the subspace $S_f = \mathrm{span}\{f(\mathcal{X})\} \subseteq \R^k$ (which, if $f_j \in \mathcal{V}$, means $S_f$ is a subspace determined by linear combinations of $\chi_i$), then $(Mf)(x)_j = f_j(x)$ for $f(x) \in S_f$. This makes Term 1 zero.

Consider the specific case $$f^*(x) = (\chi_1(x), \dots, \chi_r(x), \vec{0}_{k-r})^T$$

Let $\vec{\psi}^*(x) = f^*(x)$. As shown above, Term 2 is zero because each $\chi_j \in \mathcal{V}$ and $0 \in \mathcal{V}$.

Since $(\mathcal{K}\psi_j^*)(x) = \psi_j^*(x)$ for all $j$, Term 1 simplifies to $\mathbb{E}_x \bigl[ \|M \vec{\psi}^*(x) - \vec{\psi}^*(x)\|_2^2 \bigr]$.
We need $M \vec{\psi}^*(x) = \vec{\psi}^*(x)$ $\mu$-a.e.

Let $M^*$ be partitioned as $M^* = \begin{pmatrix} M_{11} & M_{12} \\ M_{21} & M_{22} \end{pmatrix}$, where $M_{11}$ is $r \times r$.
Then $$M^* \vec{\psi}^*(x) = \begin{pmatrix} M_{11}\vec{\chi}(x) + M_{12}\vec{0} \\ M_{21}\vec{\chi}(x) + M_{22}\vec{0} \end{pmatrix} = \begin{pmatrix} M_{11}\vec{\chi}(x) \\ M_{21}\vec{\chi}(x) \end{pmatrix}$$

For this to equal $\vec{\psi}^*(x) = \begin{pmatrix} \vec{\chi}(x) \\ \vec{0} \end{pmatrix}$, we need:
\begin{enumerate}[label=\arabic*.]
    \item $M_{11}\vec{\chi}(x) = \vec{\chi}(x)$ for all $x$. Since $\vec{\chi}(x)$ takes values $e_i \in \R^r$ (standard basis vectors) on $\mathcal{X}_i$, this implies $M_{11}e_i = e_i$ for $i=1,\dots,r$. Thus, $M_{11} = I_r$ (the $r \times r$ identity matrix).
    \item $M_{21}\vec{\chi}(x) = \vec{0}$ for all $x$. This implies $M_{21}e_i = \vec{0}$ for $i=1,\dots,r$. Thus, $M_{21} = \mathbf{0}_{(k-r) \times r}$ (the $(k-r) \times r$ zero matrix).
\end{enumerate}
The blocks $M_{12}$ (size $r \times (k-r)$) and $M_{22}$ (size $(k-r) \times (k-r)$) multiply the zero part of $\vec{\psi}^*(x)$ and thus do not affect the product $M^*\vec{\psi}^*(x)$. They can be arbitrary.
A simple choice for such an $M^*$ is $M_{11}=I_r, M_{21}=\mathbf{0}, M_{12}=\mathbf{0}, M_{22}=\mathbf{0}$ (projection onto first $r$ coordinates) or $M_{11}=I_r, M_{21}=\mathbf{0}, M_{12}=\mathbf{0}, M_{22}=I_{k-r}$ (which is $I_k$ if $r=k$). For any such $M^*$, $L(f^*, M^*) = 0$.

Now for the ``only if" part of conditions (a) and (b) for $f_j \in \mathcal{V}$, under the additional assumption that $M$ acts as identity on $f(\mathcal{X})$ for a zero-loss solution.
If $L(f,M)=0$, then conditions (a) and (b) from the theorem statement must hold:
(a) $(\mathcal{K}f_j)(x) = f_j(x_{t+\Delta})$ a.s.
(b) $(Mf)(x)_j = (\mathcal{K}f_j)(x)$ a.s.
If we further assume that for this zero-loss solution, $M$ effectively acts as an identity on the learned representations, i.e., $(Mf)(x)_j = f_j(x)$ a.s. (this is the ``simplest predictor" case where $g(z)=z$ on the manifold of $f(\mathcal{X})$), then combining with (b) gives $f_j(x) = (\mathcal{K}f_j)(x)$ a.s. So $f_j$ is an eigenfunction with eigenvalue 1.
And combining $f_j(x) = (\mathcal{K}f_j)(x)$ with (a) gives $f_j(x_t) = f_j(x_{t+\Delta})$ a.s.

By Lemma~\ref{lemma:V_properties_theory}(d) (and standard ergodic theory characterization of functions constant along trajectories within ergodic components), a function $f_j \in L^2(\mu)$ satisfying $f_j(x_{t+\Delta}) = f_j(x_t)$ a.s. must be an element of $\mathcal{V}$ (i.e., constant on each $\mathcal{X}_i$). Thus, each $f_j \in \mathcal{V}$.

If $k \ge r$, the encoder $f$ can map $\mathcal{X}$ to an $r$-dimensional subspace of $\R^k$ spanned by $r$ linearly independent combinations of $\{\chi_i\}_{i=1}^r$. For $f(x)$ to distinguish the $r$ regimes, the vector $(f_1(x), \dots, f_k(x))^T$ must take on $r$ distinct values as $x$ varies through $\mathcal{X}_1, \dots, \mathcal{X}_r$. This implies that if $f(x) = A \vec{\chi}(x)$, the $k \times r$ matrix $A$ must have rank $r$. This means $f(x)$ is an invertible linear transformation of $\vec{\chi}(x)$ onto its $r$-dimensional image in $\R^k$.
\end{proof}

\section{Synthetic Regime Generation Parameters}
\label{app:regime_params}

This section details the parameters for the $r=18$ distinct dynamical regimes used in the synthetic dataset. All master sequences have length $L_{master}=1024$. Additive observation noise (parameter noise in generation functions) was set to 0.0, meaning deterministic signal components are noise-free, and ARMA processes only contain their intrinsic process noise (typically $\epsilon_t \sim \mathcal{N}(0,1)$ before scaling by \texttt{statsmodels} parameters). Frequencies for periodic signals are defined relative to $L_{master}$ via cycle counts $c_{low}=7, c_{med}=10, c_{high}=15$, so $f_{type} = 2\pi (c_{type}/L_{master})$. Amplitudes are 1.0 unless specified. $\phi_{rand}$ denotes a random phase $\sim \mathcal{N}(0, \pi^2)$ added per master sequence for sinusoids. For trends, slopes are base values $\pm$ a $\mathcal{N}(0,1)$ component, and intercepts include a $\mathcal{N}(0,\pi^2)$ component.

\begin{itemize}
    \item \textbf{Sine\_LowFreq:} Freq. $f_{low}$.
    \item \textbf{Sine\_MedFreq:} Freq. $f_{med}$.
    \item \textbf{Sine\_HighFreq:} Freq. $f_{high}$.
    \item \textbf{Sine\_LowAmp:} Freq. $f_{med}$, Amp. $0.3$.
    \item \textbf{Sine\_Harmonics:} $0.7 \sin(f_{med} t + \phi_{rand}) + 0.3 \sin(3 f_{med} t + \phi'_{rand})$.
    \item \textbf{Trend\_Up:} Base slope $1.5$.
    \item \textbf{Trend\_Down:} Base slope $-1.5$.
    \item \textbf{AR\_PosStrong:} AR(1), $\phi_1=0.9$.
    \item \textbf{AR\_PosWeak:} AR(1), $\phi_1=0.3$.
    \item \textbf{AR\_Neg:} AR(1), $\phi_1=-0.7$.
    \item \textbf{MA\_Pos:} MA(1), $\theta_1=0.7$.
    \item \textbf{ARMA\_Mixed:} ARMA(1,1), $\phi_1=0.5, \theta_1=-0.4$.
    \item \textbf{Square\_LowFreq:} Period $L_{master}/c_{low}$.
    \item \textbf{Square\_HighFreq:} Period $L_{master}/c_{high}$.
    \item \textbf{Sawtooth\_MedFreq:} Period $L_{master}/c_{med}$, initial phase randomized.
    \item \textbf{Pulses\_Sparse:} Approx. 5 pulses, width $L_{master}/50$, Amp. $2.0$.
    \item \textbf{Sine\_Trend:} $0.8 \sin(f_{med} t + \phi_{rand}) + \text{Trend (base slope 1.0)}$.
    \item \textbf{Sine\_HighNoise:} Sinusoid with $f_{med}$, internal process noise std. dev. approx. $3 \times$ that of ARMA processes. 
\end{itemize}
All ARMA processes are generated using \texttt{statsmodels.tsa.arima\_process.ArmaProcess} \cite{Seabold2010}, ensuring stationarity for the chosen parameters.


\section{Model Architecture Details}
\label{app:model_architectures}

This section provides detailed architectural specifications for the encoder and predictor networks used in the JEPA models.

\subsubsection{Convolutional Encoder (\texttt{ConvEmbedder})}
The online encoder $\fenc$ and target encoder $\fEMA$ utilize a 1D Convolutional Neural Network architecture with the following structure:

\begin{table}[ht]
\centering
\caption{Architecture of the 1D Convolutional Encoder ($f_\theta$). Input sequence length $n_c=768$.}
\label{tab:encoder_arch}
\resizebox{\columnwidth}{!}{%
  \begin{tabular}{@{}lccccc@{}}
    \toprule
    Layer Type       & In Channels & Out Channels & Kernel Size & Stride & Padding \\ \midrule
    Conv1D + ReLU    & 1           & 16           & 7           & 2      & 3       \\
    Conv1D + ReLU    & 16          & 32           & 5           & 2      & 2       \\
    Conv1D + ReLU    & 32          & 64           & 3           & 2      & 1       \\
    Conv1D + ReLU    & 64          & 128          & 3           & 2      & 1       \\
    \midrule
    Flatten          & \multicolumn{5}{l}{Output from last Conv1D layer is flattened.} \\
    \midrule
    Linear    & [Calculated] &  $2 \times k$ & --          & --     & --      \\ 
    \bottomrule
  \end{tabular}%
}
\end{table}

The [Calculated] input features to the first linear layer depend on the output dimensions of the final convolutional layer. The latent dimension $k$ was set to 32.

\subsubsection{MLP Predictor (\texttt{MLPPredictor})}
The non-linear MLP predictor $\fpred$ has the following structure:
\begin{table}[htbp]
\centering
\caption{Architecture of the MLP Predictor ($g_\phi$). Input dimension is $k$.}
\label{tab:mlp_predictor_arch}
\begin{tabular}{@{}lcc@{}}
\toprule
Layer Type       & Input Features & Output Features \\ \midrule
Linear + ReLU    & $k$            & $2 \times k$  \\
Linear + ReLU    & $2 \times k$ &  $2 \times k$ \\ 
Linear           &  $2 \times k$ & $k$ \\ \bottomrule
\end{tabular}
\end{table}

The hidden dimension multiplier was 2 and the number of hidden layers was 1.

\subsection{Linear Predictor (\texttt{LinearPredictor})}
The linear predictor $\fpred$ consists of a single linear layer: $g(z) = Mz$, where $M \in \R^{k \times k}$. The bias term was set to false. For specific experiments, $M$ was initialized as an identity matrix.

\end{document}